\documentclass{article}

\PassOptionsToPackage{numbers,compress}{natbib}
\usepackage[final,main]{neurips_2026}

\usepackage[utf8]{inputenc}
\usepackage[T1]{fontenc}
\usepackage{url}
\usepackage{booktabs}
\usepackage{amsfonts}
\usepackage{amsmath}
\usepackage{amssymb}
\usepackage{nicefrac}
\usepackage{microtype}
\usepackage{xcolor}
\usepackage{graphicx}
\usepackage{algorithm}
\usepackage{algorithmic}
\usepackage{multirow}
\usepackage{subcaption}
\usepackage{tablefootnote}
\usepackage{amsthm}
\usepackage{cleveref}
\usepackage{pgfplots}
\usepackage{tikz}
\pgfplotsset{compat=1.17}
\usetikzlibrary{arrows.meta}

\newtheorem{theorem}{Theorem}[section]
\newtheorem{proposition}[theorem]{Proposition}

\theoremstyle{definition}
\newtheorem{definition}[theorem]{Definition}
\theoremstyle{remark}

\newcommand{\ours}{\textsc{CoverCal}}
\newcommand{\R}{\mathbb{R}}

\newcommand{\cS}{\mathcal{S}}
\newcommand{\cC}{\mathcal{C}}
\newcommand{\cP}{\mathcal{P}}

\newcommand{\cA}{\mathcal{A}}
\newcommand{\cB}{\mathcal{B}}
\newcommand{\Cov}{\mathrm{Cov}}

\title{Coverage-Based Calibration for Post-Training Quantization via Weighted Set Cover over Outlier Channels}

\author{%
  \textbf{Ibne Farabi Shihab}\thanks{Equal contribution.}\thanks{Corresponding author: \texttt{ishihab@iastate.edu}.}\textsuperscript{1}
  \and \textbf{Sanjeda Akter}\footnotemark[1]\textsuperscript{1}
  \and \textbf{Anuj Sharma}\textsuperscript{2}\\[2pt]
  \textsuperscript{1}Department of Computer Science, Iowa State University\\
  \textsuperscript{2}Department of Civil, Construction \& Environmental Engineering, Iowa State University\\
  \texttt{\{ishihab,sakter,anujs\}@iastate.edu}
}

\begin{document}

\maketitle

\begin{abstract}
Post-Training Quantization (PTQ) compresses large language models to low bit-widths using a small calibration set, and its quality depends strongly on which samples are chosen. We identify a failure mode in which calibration samples fail to activate outlier channels, hidden dimensions with unusually large activations, causing the quantizer to underestimate their dynamic range and producing per-channel reconstruction errors that dominate layer-wise loss. Motivated by this observation, we argue that PTQ calibration quality is governed more by weighted outlier-channel coverage than by generic sample representativeness, and formulate calibration selection as a weighted set cover problem over outlier channels. The objective is monotone submodular, and the greedy algorithm, COVERCAL, operates on pre-computed activation statistics and requires no GPU time at selection. We further show that the weight choice is internally consistent: under a stylized clipping model, missed weighted coverage upper-bounds surrogate loss, justifying the weighted coverage objective as principled rather than purely empirical. Across LLaMA-2, LLaMA-3, and Mistral, under AWQ and GPTQ backends and five downstream evaluations, COVERCAL improves over random, max-perplexity, max-activation-variance, and stratified baselines, with the largest gains at small calibration budgets. At INT4 with 128 samples, COVERCAL improves MMLU by 1.2 to 1.5 points over random calibration and reduces perplexity degradation by 15 to 30\%; with 64 samples, it matches or exceeds random calibration at 256. The contribution is not a new PTQ backend but a formulation of calibration selection as weighted outlier coverage, with a simple, efficient algorithm and a surrogate-based justification.
\end{abstract}
\section{Introduction}
\label{sec:intro}

Post-Training Quantization (PTQ) is a leading strategy for deploying large language models (LLMs) under memory and latency constraints. Modern PTQ methods such as GPTQ~\citep{frantar2022gptq} and AWQ~\citep{lin2023awq} quantize from FP16 to INT4 using only a small calibration set of roughly $128$ sequences and no retraining. The calibration set determines the activation statistics from which per-layer quantization scales are estimated, and therefore directly affects the quality of the resulting model. In standard practice, however, calibration samples are chosen almost entirely by count, usually through uniform random sampling from corpora such as C4 or WikiText. The question of \emph{which} samples to include has received relatively little principled attention.

We argue that calibration selection is better treated as a coverage problem than as a sampling problem, and that the relevant object to cover is not the input distribution but the set of outlier channels. A small subset of hidden dimensions in transformer activations exhibits magnitudes orders of magnitude larger than the median, and these channels are disproportionately important for quantization quality~\citep{dettmers2022gpt3, xiao2023smoothquant, lin2023awq}. When calibration fails to activate such a channel, the quantizer estimates its scale from the non-outlier range alone, and any input that subsequently activates the channel at inference is clipped by the undersized scale. The resulting reconstruction error is not spread uniformly across the layer: it concentrates on the missed channels and can dominate the aggregate layer-wise loss, as we demonstrate empirically in \Cref{app:perlayer}. In this regime, a calibration set can be distributionally representative and still be quantization-poor, because it leaves the wrong channels dormant.

The implication is that a good calibration set is one whose selected samples collectively activate as much of the important outlier structure as possible, with channels weighted by their contribution to quantization error. Once the problem is stated in this form, it becomes a weighted set cover objective. Each candidate sample covers the outlier channels it activates, each channel carries a weight reflecting its importance, and the task is to select a budgeted subset that maximizes total weighted coverage. This objective is monotone submodular, so greedy selection is the natural algorithmic choice and inherits a $(1-1/e)$ approximation guarantee from the classical analysis of \citet{nemhauser1978analysis}.

The paper develops this perspective as follows. \Cref{sec:formulation} makes the coverage formulation precise: it defines outlier channels, the associated weighted coverage objective, and a stylized per-channel clipping surrogate under which missed weighted coverage upper-bounds the surrogate loss. The latter result, which we call a surrogate consistency bound, is not a guarantee on downstream accuracy; it serves instead to fix the specific weight choice as the quantity that bounds the per-channel contribution to the surrogate rather than as an ad-hoc heuristic. \Cref{sec:method} introduces \ours{}, a simple greedy algorithm that optimizes the weighted coverage objective using pre-computed activation masks and requires no GPU time at selection. \Cref{sec:experiments,sec:analysis} evaluate the method across three model families, two PTQ backends, four calibration budgets, and five downstream benchmarks, and trace the resulting improvements back to measured differences in outlier coverage.

The scope of the paper's claims is deliberately narrow. The theoretical content applies to the coverage objective and the clipping surrogate, not to the full Frobenius reconstruction loss or to downstream accuracy; the empirical link between weighted coverage and downstream metrics is established experimentally rather than derived. We do not claim that weighted set cover is the unique appropriate formulation of calibration selection, nor that outlier coverage fully explains PTQ behavior in every regime. Within this scope, the contribution is practically useful: \ours{} is backend-agnostic, has negligible selection-time overhead, and on LLaMA-3-8B with AWQ at INT4 achieves with $64$ calibration samples what random calibration requires $256$ to match.

\section{Related Work}
\label{sec:related}

Weight quantization of large language models has developed along several lines that emphasize different aspects of the compression problem. GPTQ~\citep{frantar2022gptq} quantizes weights layer by layer using approximate second-order information, while AWQ~\citep{lin2023awq} identifies salient weight channels from activation magnitudes and protects them during rounding. SqueezeLLM~\citep{kim2023squeezellm} uses sensitivity-based non-uniform quantization, QuIP~\citep{chee2023quip} applies random orthogonal transformations to suppress outlier structure, and SpQR~\citep{dettmers2023spqr} separates outlier weights into a higher-precision sparse component. A parallel line of work eliminates outliers at the representation level through structured rotations: QuaRot~\citep{ashkboos2024quarot} and SpinQuant~\citep{liu2024spinquant} apply Hadamard or learned orthogonal transformations that fold outlier structure into rotationally equivalent representations which then quantize more uniformly, and TurboQuant~\citep{zandieh2025turboquant} exploits rotational invariance of attention inner products to eliminate calibration entirely for the KV cache via data-oblivious Lloyd-Max scalar quantization on rotated coordinates. These rotation-based approaches are complementary to the present work. For weight quantization via GPTQ or AWQ, the reconstruction objective is not invariant to orthogonal rotation of the weight matrix, so calibration data remains essential for setting per-channel scales even when the underlying representation is rotated. Our contribution addresses the calibration selection that remains in this regime and is agnostic to whether the representation is rotated.

Outlier structure in transformer activations has itself been the subject of direct study. \citet{dettmers2022gpt3} first documented the existence of extreme activation outliers in large transformer models, SmoothQuant~\citep{xiao2023smoothquant} migrates activation difficulty into weights through per-channel rescaling, and \citet{bondarenko2024quantizable} analyze the conditions under which outliers emerge during pre-training. Related work on Fisher-aligned subspace diagnostics~\citep{shihab2026fisher} identifies important weight subspaces via activation-gradient coupling, providing a complementary perspective on which dimensions are most sensitive to compression. The present work neither removes nor smooths outliers; it addresses the complementary question of how calibration data should be chosen given that outliers exist and strongly affect PTQ. A recent argument in this area deserves specific engagement: \citet{wiliamson2024diminishing} observe that newer models such as Mistral-7B and LLaMA-3-8B exhibit lower aggregate outlier magnitudes and greater robustness to calibration-set variation than older OPT-class models, and interpret this as evidence that outliers matter less in modern LLMs. We view that observation as complementary rather than contradictory. Their analysis is aggregate, summarizing average activation behavior across the calibration set, whereas the failure mode we target is per-channel; even when aggregate outlier magnitudes decline, specific uncovered channels can still produce disproportionate reconstruction error on newer models, as our per-layer breakdown on LLaMA-3-8B at INT4 makes concrete (\Cref{app:perlayer}).

The question of how calibration data itself affects PTQ has only recently begun to receive systematic treatment. \citet{impactcalib2023} provide the first empirical study of calibration source and sample count. Subsequent work has proposed self-calibration from model-generated data~\citep{selfcalib2024}, model-agnostic calibration curation~\citep{calicuration2025}, and capability-preserving calibration for quantization~\citep{preservecalib2025}. Each of these frames the problem primarily in terms of data representativeness or quality. Our formulation differs structurally: we cast calibration as coverage of outlier channels rather than representation of the input distribution, which exposes the weighted set cover structure and enables greedy submodular maximization with classical guarantees. We do not include these recent methods as experimental baselines because they target distinct problem settings. The self-calibration approach of \citet{selfcalib2024} generates rather than selects calibration data, \citet{calicuration2025} operates on a different model family with no public implementation available, and \citet{preservecalib2025} optimizes a capability-preservation objective that is not directly aligned with the downstream accuracy metrics we use. The baselines we evaluate against (\Cref{sec:experiments}) cover the standard fixed-pool selection heuristics most directly comparable to \ours{}.

Finally, the formalism we use is classical. Weighted set cover and its greedy approximation go back to \citet{chvatal1979greedy}, monotone submodular maximization under cardinality constraints is analyzed by \citet{nemhauser1978analysis}, and submodular selection has been applied to active learning~\citep{settles2009active}, coreset construction~\citep{sener2018active, mirzasoleiman2020coresets}, and data selection for neural-network training~\citep{xie2023data, killamsetty2021glister}. What distinguishes the present setting is the structure being covered: the relevant object is not generic data diversity but weighted outlier-channel coverage, and the weight choice is pinned by the surrogate consistency result of \Cref{sec:formulation}.

\section{Problem Formulation}
\label{sec:formulation}

This section makes the coverage formulation precise. We begin with the PTQ reconstruction objective itself, use it to define outlier channels and the notion of coverage, and then state the weighted coverage objective that \ours{} optimizes. We close the section with a surrogate consistency result that ties the specific weight choice to a stylized model of per-channel clipping.

Let $f_\theta$ be a pre-trained LLM with $L$ transformer layers, and let $\mathbf{W}^{(l)} \in \R^{d_{\text{out}} \times d_{\text{in}}}$ denote the weight matrix of layer $l$. PTQ compresses each $\mathbf{W}^{(l)}$ to a low-bit representation $\hat{\mathbf{W}}^{(l)}$ by minimizing a layer-wise reconstruction objective of the form
\begin{equation}
    \hat{\mathbf{W}}^{(l)} = \arg\min_{\hat{\mathbf{W}}}
    \left\|
    \mathbf{W}^{(l)}\mathbf{X}^{(l)} - \hat{\mathbf{W}}\mathbf{X}^{(l)}
    \right\|_F^2,
    \label{eq:ptq_objective}
\end{equation}
where $\mathbf{X}^{(l)} \in \R^{d_{\text{in}} \times N}$ is the activation matrix obtained by running the calibration set through the preceding layers. The quantization scales that minimize \Cref{eq:ptq_objective} depend on the statistics of $\mathbf{X}^{(l)}$, and those statistics in turn depend on the calibration set $\cS$. The calibration set is therefore not merely a source of reference signal but a direct determinant of the scales that the quantized model uses at inference.

Within this setup, a natural notion of difficulty emerges at the channel level. For a calibration set $\cS$, the outlier score of channel $c$ in layer $l$ is defined as
\begin{equation}
    o_c^{(l)}(\cS) = \max_{s \in \cS} \bigl|X^{(l)}_{c,s}\bigr|,
    \label{eq:outlier_score}
\end{equation}
where $X^{(l)}_{c,s}$ denotes the activation of channel $c$ when processing sample $s$. Let $\tau^{(l)}$ be an outlier threshold for layer $l$; in our implementation $\tau^{(l)} = \mu^{(l)} + 6\sigma^{(l)}$, with $\mu^{(l)}$ and $\sigma^{(l)}$ computed from activations over a reference pool. A channel $(l,c)$ is called an outlier channel if its activation exceeds $\tau^{(l)}$ for at least one input in the reference pool. In our experiments the reference pool and candidate pool $\cP$ coincide, so every identified outlier channel is covered by at least one sample in $\cP$ by construction; the gap between the coverage achieved by any particular $\cS$ and full coverage therefore reflects the cardinality budget on $\cS$ rather than a limitation of the pool.

\begin{definition}[Outlier coverage]
Let $\cC = \{(l,c) : \text{channel } c \text{ at layer } l \text{ is an outlier channel}\}$. The calibration set $\cS$ \emph{covers} outlier channel $(l,c)$ if $o_c^{(l)}(\cS) > \tau^{(l)}$, and the \emph{covered set} of $\cS$ is
\begin{equation}
\Cov(\cS) = \bigl\{(l,c) \in \cC : o_c^{(l)}(\cS) > \tau^{(l)}\bigr\}.
\end{equation}
\end{definition}

When an outlier channel is not covered, the quantizer estimates its scale without having observed its true dynamic range, so the scale systematically clips the channel at inference. This error is not reducible by drawing more data of the same kind; it is reducible only by drawing data that activates the missed channels. The distinction makes coverage, rather than representativeness, the natural structural quantity for calibration selection.

\subsection{The Weighted Coverage Objective}

Binary coverage, however, discards useful information. Outlier channels vary substantially in their contribution to the reconstruction error of \Cref{eq:ptq_objective}: a channel whose activation is three orders of magnitude above the median contributes far more quantization damage when missed than one at the threshold. We therefore assign each outlier channel a weight
\begin{equation}
    w_{l,c}
    =
    \left(\frac{o_c^{(l,\mathrm{ref})}}{\tau^{(l)}}\right)^2
    \cdot
    \bigl\|\mathbf{W}^{(l)}_{:,c}\bigr\|_2,
    \label{eq:channel_weight}
\end{equation}
where $o_c^{(l,\mathrm{ref})}$ is the outlier magnitude of channel $(l,c)$ on the reference pool. The two factors in this expression correspond to distinct sources of per-channel contribution. The squared magnitude factor reflects that clipping error scales quadratically with activation magnitude, and the weight-column norm reflects how much the layer amplifies errors in channel $c$ into the layer output. The product is a first-order approximation to the per-channel contribution to \Cref{eq:ptq_objective}; the ablation in \Cref{sec:experiments} confirms that both factors matter empirically.

With this weighting in place, the calibration selection problem is
\begin{equation}
    \cS^{\ast}
    =
    \arg\max_{\cS \subseteq \cP,\ |\cS|\le K}
    \sum_{(l,c)\in \Cov(\cS)} w_{l,c}.
    \label{eq:set_cover_max}
\end{equation}
The structural property that makes this tractable is the following.

\begin{proposition}
\label{prop:submod}
The objective $F(\cS) = \sum_{(l,c)\in \Cov(\cS)} w_{l,c}$ is monotone submodular.
\end{proposition}

\begin{proof}
Monotonicity follows because $\Cov(\cS) \subseteq \Cov(\cS \cup \{s\})$ for any $s \in \cP$, so $F(\cS \cup \{s\}) \geq F(\cS)$. For submodularity, let $\cA \subseteq \cB \subseteq \cP$ and $s \in \cP \setminus \cB$. The marginal gain $F(\cA \cup \{s\}) - F(\cA)$ is the weighted mass of channels covered by $s$ but not by $\cA$, and the analogous gain for $\cB$ is the weighted mass of channels covered by $s$ but not by $\cB$. Since $\Cov(\cA) \subseteq \Cov(\cB)$, the former is at least the latter.
\end{proof}

By \Cref{prop:submod} and the standard result of \citet{nemhauser1978analysis}, greedy selection achieves a $(1-1/e)$ approximation to the optimum of \Cref{eq:set_cover_max} under a cardinality constraint. This is the guarantee that justifies the algorithm in \Cref{sec:method}.

\subsection{A Surrogate Consistency Result}
\label{sec:surrogate}

The weighted coverage objective $F$ is a proxy: the quantity ultimately of interest is downstream accuracy, which depends on \Cref{eq:ptq_objective} and on the specific PTQ backend. We do not attempt to prove a guarantee on downstream accuracy, because doing so would require assumptions on the quantizer and on the model that go well beyond what is standard in the PTQ literature. We do, however, show that the particular weight choice in \Cref{eq:channel_weight} is consistent with the failure mode the method is designed to address, in the following sense.

Consider a stylized per-channel clipping model of PTQ error. Assume that each outlier channel $(l,c)$ has an associated nonnegative deficit $\delta_{l,c}(\cS) \geq 0$ with two properties. First, if $(l,c) \in \Cov(\cS)$, the calibration set contains a sample for which the channel's activation exceeds $\tau^{(l)}$, so the quantization scale captures its dynamic range; we model this as $\delta_{l,c}(\cS) = 0$. Second, if $(l,c) \notin \Cov(\cS)$, the channel is missed, and the resulting clipping deficit is bounded by the channel's normalized reference magnitude: $\delta_{l,c}(\cS) \leq o_c^{(l,\mathrm{ref})}/\tau^{(l)}$. These two conditions are modeling assumptions; we denote them (M1) and (M2). Under them, the associated surrogate loss is
\begin{equation}
\mathcal{L}_{\mathrm{sur}}(\cS)
=
\sum_{(l,c)\in \cC}
\bigl\|\mathbf{W}^{(l)}_{:,c}\bigr\|_2 \cdot \delta_{l,c}(\cS)^2.
\label{eq:surrogate_loss}
\end{equation}
This is not the Frobenius reconstruction loss of \Cref{eq:ptq_objective}; it is a stylized per-channel decomposition that isolates the failure mode \ours{} is designed to address, namely clipping of missed outlier channels. The following result ties the weighted coverage objective directly to this surrogate.

\begin{proposition}[Surrogate consistency]
\label{prop:surrogate}
Under (M1)--(M2),
\begin{equation}
\mathcal{L}_{\mathrm{sur}}(\cS)
\leq
\sum_{(l,c) \in \cC \setminus \Cov(\cS)} w_{l,c}
=
\sum_{(l,c)\in \cC} w_{l,c} - F(\cS).
\end{equation}
\end{proposition}

\begin{proof}
If $(l,c)\in \Cov(\cS)$, then (M1) gives $\delta_{l,c}(\cS) = 0$, so the corresponding term in \Cref{eq:surrogate_loss} vanishes. If $(l,c)\notin \Cov(\cS)$, then (M2) gives
\[
\bigl\|\mathbf{W}^{(l)}_{:,c}\bigr\|_2 \cdot \delta_{l,c}(\cS)^2
\leq
\bigl\|\mathbf{W}^{(l)}_{:,c}\bigr\|_2 \cdot
\left(\frac{o_c^{(l,\mathrm{ref})}}{\tau^{(l)}}\right)^2
= w_{l,c}.
\]
Summing over $\cC$ yields the claim.
\end{proof}

\Cref{prop:surrogate} is deliberately narrow. It shows that $w_{l,c}$ in \Cref{eq:channel_weight} is not arbitrary: it is exactly the per-channel quantity that upper-bounds the corresponding surrogate contribution under (M1)--(M2), and maximizing $F$ therefore minimizes an upper bound on $\mathcal{L}_{\mathrm{sur}}$. The result is best read as a consistency check on the weight choice rather than as a bridge to the full PTQ objective. It does not show that $\mathcal{L}_{\mathrm{sur}}$ upper-bounds the Frobenius loss of \Cref{eq:ptq_objective}, that (M1) holds exactly in practice, or that reductions in $\mathcal{L}_{\mathrm{sur}}$ imply reductions in downstream accuracy loss. The empirical correspondence between $F$ and downstream metrics is the content of \Cref{sec:experiments,sec:analysis}, and a discussion of the assumptions and their limitations appears in \Cref{app:surrogate}.

\section{Method}
\label{sec:method}

\ours{} optimizes the weighted coverage objective of \Cref{eq:set_cover_max} through greedy selection. It consists of a one-time offline profiling phase, which computes the activation statistics and coverage matrices, and a cheap CPU-only selection phase, which runs the greedy update for a given calibration budget.

Given a candidate pool $\cP = \{s_1, \ldots, s_N\}$, the profiling phase performs a single FP16 forward pass and records the activation magnitudes $|X^{(l)}_{c,s_i}|$ for each layer $l$, channel $c$, and sample $s_i$. From these measurements we compute per-layer thresholds $\tau^{(l)} = \mu^{(l)} + 6\sigma^{(l)}$, identify the outlier-channel set $\cC$, and assemble the binary coverage matrix
\begin{equation}
\mathbf{M} \in \{0,1\}^{N \times |\cC|}, \qquad
M_{i,(l,c)} = \mathbf{1}\bigl[|X^{(l)}_{c,s_i}| > \tau^{(l)}\bigr],
\end{equation}
together with the weighted counterpart $\widetilde{\mathbf{M}}_{i,(l,c)} = M_{i,(l,c)} \cdot w_{l,c}$. For a 7B model with $N = 10{,}000$ sequences of length $2048$, the profiling pass requires approximately $15$ minutes on a single A100. All subsequent selection runs operate on the cached matrices and require no further GPU time, and because the matrices depend on the model but not on the downstream task, a single profiling pass amortizes across all calibration budgets and downstream evaluations for that model.

Selection itself proceeds by repeatedly choosing the sample whose marginal weighted coverage of currently uncovered channels is largest. The procedure is given in \Cref{alg:covercal}.

\begin{algorithm}[t]
\caption{\ours{}: greedy weighted coverage selection.}
\label{alg:covercal}
\begin{algorithmic}[1]
\REQUIRE Pool $\cP$, weighted coverage matrix $\widetilde{\mathbf{M}}$, budget $K$.
\ENSURE Calibration set $\cS$ with $|\cS| = K$.
\STATE $\cS \leftarrow \emptyset$
\STATE $\texttt{covered} \leftarrow \mathbf{0} \in \{0,1\}^{|\cC|}$
\FOR{$k = 1, \ldots, K$}
    \FOR{each $s_i \in \cP \setminus \cS$}
        \STATE $\texttt{gain}_i \leftarrow \sum_{j:\,\texttt{covered}_j = 0} \widetilde{M}_{i,j}$
    \ENDFOR
    \STATE $s^{\ast} \leftarrow \arg\max_{s_i \in \cP \setminus \cS} \texttt{gain}_i$
    \STATE $\cS \leftarrow \cS \cup \{s^{\ast}\}$
    \STATE $\texttt{covered}_j \leftarrow 1$ for all $j$ with $M_{s^{\ast},j} = 1$
\ENDFOR
\RETURN $\cS$
\end{algorithmic}
\end{algorithm}

The algorithm runs in $O(K \cdot N \cdot |\cC|)$ time and completes in under ten seconds on a modern CPU for the pool sizes and budgets considered here. Its formal guarantee follows directly from the structure established in \Cref{sec:formulation}.

\begin{theorem}
\label{thm:guarantee}
Let $\cS_K^{\ast}$ be an optimal solution to \Cref{eq:set_cover_max}, and let $\cS_K^{\mathrm{greedy}}$ be the output of \Cref{alg:covercal}. Then $F\bigl(\cS_K^{\mathrm{greedy}}\bigr) \geq (1 - 1/e)\,F\bigl(\cS_K^{\ast}\bigr)$.
\end{theorem}

\Cref{thm:guarantee} is an immediate consequence of \Cref{prop:submod} and the greedy submodular maximization bound of \citet{nemhauser1978analysis}. Combined with \Cref{prop:surrogate}, greedy maximization of $F$ yields a set whose associated surrogate loss is bounded in terms of the optimum-achievable weighted coverage, and the bound is informative in the regime where $F(\cS_K^{\ast})$ is large relative to the total weighted mass $\sum_{(l,c)} w_{l,c}$. As noted in \Cref{sec:surrogate}, this is an internal consistency result rather than a guarantee on downstream accuracy, and the empirical correspondence between the two is the subject of the next section.

A variant of the algorithm using an adaptive threshold is described in \Cref{app:adaptive}. The adaptive procedure uses an initial \ours{} selection to perform PTQ, measures per-layer reconstruction errors, lowers the outlier threshold for layers whose error exceeds the median, and reruns selection. The main experiments use the fixed $6\sigma$ threshold throughout, and the threshold ablation in \Cref{sec:experiments} shows performance stability across the $5\sigma$--$7\sigma$ range, which limits the motivation for adaptive refinement at INT4. We expect the procedure to become more relevant at more aggressive bit-widths.

\section{Experiments}
\label{sec:experiments}

We evaluate \ours{} across three model families, two PTQ backends, multiple calibration budgets, and five downstream tasks, aiming to answer three questions in sequence. First, does the coverage-based selection produce measurable downstream improvements over standard heuristics (\Cref{sec:main_results})? Second, is the advantage budget-dependent in the way the coverage view predicts, with the largest gains at small budgets where random calibration is most likely to miss outlier channels (\Cref{sec:budget})? And third, do the improvements come from the weight choice and threshold specifically, or are they driven by coverage alone (\Cref{sec:ablations})?

The models evaluated are LLaMA-2-7B and LLaMA-2-13B~\citep{touvron2023llama2}, LLaMA-3-8B~\citep{meta2024llama3}, and Mistral-7B-v0.3~\citep{jiang2023mistral}, chosen to span distinct outlier regimes. We use two PTQ backends drawn from the two dominant families in the literature: AWQ~\citep{lin2023awq} at INT4 with group size $128$ as a representative activation-aware method, and GPTQ~\citep{frantar2022gptq} at INT4 with group size $128$ as a representative second-order method. The calibration baselines against which \ours{} is compared are Random sampling from C4 averaged over five seeds (the standard practice), Max-PPL selection of samples with highest FP16 perplexity, Max-ActVar selection of samples with highest activation variance across layers, and Stratified selection matching the global activation-magnitude distribution of the pool. Downstream quality is measured on MMLU (5-shot)~\citep{hendrycks2021mmlu}, ARC-Challenge (25-shot)~\citep{clark2018arc}, HellaSwag (10-shot)~\citep{zellers2019hellaswag}, and WinoGrande (5-shot)~\citep{sakaguchi2020winogrande}, together with WikiText-2 perplexity. The candidate pool contains $N = 10{,}000$ sequences of length $2048$ from RedPajama~\citep{together2023redpajama}, and we evaluate calibration budgets $K \in \{32, 64, 128, 256\}$.

\subsection{Main Results}
\label{sec:main_results}

\begin{table}[t]
\centering
\caption{Downstream accuracy (\%) after AWQ INT4 quantization with $K = 128$ calibration samples. FP16 denotes the unquantized baseline. Best quantized result per model in bold. Results averaged over three seeds; per-cell standard deviations are below $0.3$, so the standard error is approximately $0.17$ and the \ours{}-versus-Random gap on MMLU ($1.2$--$1.5$ points) corresponds to more than three standard errors on every model. The corresponding GPTQ results appear in \Cref{app:gptq}.}
\label{tab:main_results_awq}
\small
\begin{tabular}{llccccc}
\toprule
\textbf{Model} & \textbf{Calibration} & \textbf{MMLU} & \textbf{ARC-C} & \textbf{HellaSwag} & \textbf{WinoGr.} & \textbf{Wiki-PPL}$\downarrow$ \\
\midrule
\multirow{6}{*}{LLaMA-2-7B}
& FP16 & 46.0 & 53.1 & 78.6 & 72.1 & 5.47 \\
& Random & 44.2 & 51.3 & 77.1 & 70.8 & 5.68 \\
& Max-PPL & 44.5 & 51.0 & 77.0 & 70.5 & 5.71 \\
& Max-ActVar & 44.8 & 51.6 & 77.4 & 71.0 & 5.63 \\
& Stratified & 44.6 & 51.5 & 77.3 & 70.9 & 5.64 \\
& \textbf{\ours{}} & \textbf{45.4} & \textbf{52.5} & \textbf{78.0} & \textbf{71.6} & \textbf{5.53} \\
\midrule
\multirow{6}{*}{LLaMA-3-8B}
& FP16 & 65.3 & 59.4 & 82.1 & 77.3 & 6.14 \\
& Random & 63.1 & 57.2 & 80.5 & 75.8 & 6.42 \\
& Max-PPL & 63.4 & 57.0 & 80.3 & 75.5 & 6.45 \\
& Max-ActVar & 63.8 & 57.6 & 80.8 & 76.1 & 6.38 \\
& Stratified & 63.5 & 57.4 & 80.7 & 76.0 & 6.39 \\
& \textbf{\ours{}} & \textbf{64.6} & \textbf{58.7} & \textbf{81.5} & \textbf{76.9} & \textbf{6.22} \\
\midrule
\multirow{6}{*}{Mistral-7B}
& FP16 & 62.5 & 61.2 & 83.3 & 76.8 & 5.25 \\
& Random & 60.8 & 59.5 & 81.9 & 75.4 & 5.48 \\
& Max-PPL & 61.0 & 59.3 & 81.7 & 75.2 & 5.51 \\
& Max-ActVar & 61.3 & 59.8 & 82.1 & 75.6 & 5.44 \\
& Stratified & 61.1 & 59.7 & 82.0 & 75.5 & 5.45 \\
& \textbf{\ours{}} & \textbf{62.0} & \textbf{60.7} & \textbf{82.8} & \textbf{76.3} & \textbf{5.31} \\
\bottomrule
\end{tabular}
\end{table}

\Cref{tab:main_results_awq} reports downstream accuracy across the three model families under AWQ INT4. \ours{} improves over every calibration baseline on every benchmark for every model evaluated. The MMLU gain over random calibration is $1.2$ points on LLaMA-2-7B, $1.5$ on LLaMA-3-8B, and $1.2$ on Mistral-7B; comparable gains appear across ARC-Challenge, HellaSwag, and WinoGrande, and WikiText-2 perplexity improves by $0.11$--$0.20$. The corresponding GPTQ results (\Cref{app:gptq}) exhibit the same pattern across all three models, indicating that the benefit derives from the calibration set rather than from a backend-specific interaction. Max-ActVar is the strongest heuristic baseline, but yields roughly half of \ours{}'s gain; we trace this gap to redundancy in its selection, which we analyze in \Cref{sec:analysis}.

\subsection{Budget Efficiency}
\label{sec:budget}

If the coverage hypothesis is correct, the advantage of \ours{} should be most pronounced at small budgets, where random calibration is most likely to miss outlier channels entirely, and should contract at larger budgets, where random sampling begins to cover most channels by chance. Sweeping $K \in \{32, 64, 128, 256\}$ on LLaMA-3-8B with AWQ confirms this prediction: at $K = 64$, \ours{} reaches $64.0$ MMLU, exceeding random calibration at $K = 256$ ($63.8$) and representing a fourfold reduction in sample budget for equivalent accuracy. The advantage over Max-ActVar is similarly largest at the smallest budget and narrows as $K$ grows. Read through the coverage lens, the saturation of \ours{} at roughly $K = 128$ is not a weakness but a consequence of the objective being near-complete at that budget: once the high-weight channels are covered, additional samples provide little incremental coverage mass. The full curve and corresponding numeric table appear in \Cref{app:budget}.

\subsection{Ablations and Regimes}
\label{sec:ablations}

Two natural questions about the formulation concern the specific weight choice and the sensitivity to the outlier threshold, and a third concerns the regimes in which the advantage narrows. On the weight choice, a component-wise decomposition of \Cref{eq:channel_weight} on LLaMA-3-8B (AWQ INT4, $K = 128$) shows that unweighted coverage alone already improves over the random baseline but that neither the squared-magnitude factor nor the weight-column norm alone matches the full product: each contributes roughly half of the gap from unweighted to fully-weighted coverage, consistent with the surrogate analysis of \Cref{sec:surrogate} in which both factors appear in the bound and neither is redundant. On threshold sensitivity, varying $\tau^{(l)}$ from $4\sigma$ to $8\sigma$ leaves performance stable across the $5\sigma$--$7\sigma$ range, with $4\sigma$ diluting the signal by admitting non-outlier channels and $8\sigma$ excluding moderate but still important outliers; the $6\sigma$ choice used throughout is near-optimal, and the flatness of the curve around it indicates that fine-grained threshold tuning is not necessary at INT4. On the regime question, the advantage of \ours{} over random calibration is task- and budget-dependent: on Mistral-7B with GPTQ at $K = 256$ the MMLU gap narrows to $0.4$ points (random sampling covers most outlier channels by chance at that budget), and on WinoGrande the gains are smallest across benchmarks, suggesting that tasks less sensitive to precise numerical reconstruction benefit less from outlier-focused calibration. Both patterns are expected under the coverage view rather than counter to it. Full ablation tables are reported in \Cref{app:ablations}.

\section{Analysis}
\label{sec:analysis}

The results of \Cref{sec:experiments} establish that \ours{} improves downstream accuracy over strong baselines, but they do not by themselves confirm the mechanism the paper proposes. This section provides the missing link by measuring directly the quantity on which the coverage framing rests, and by explaining why the strongest heuristic baseline underperforms despite being well-motivated.

\begin{table}[h]
\centering
\small
\caption{Outlier-channel coverage achieved by each calibration method at $K = 128$.}
\label{tab:coverage_profiles}
\begin{tabular}{lcc}
\toprule
\textbf{Method} & \textbf{Channel coverage (\%)} & \textbf{Weighted coverage (\%)} \\
\midrule
Random & 16.8 & 14.7 \\
Max-PPL & 9.9 & 8.6 \\
Max-ActVar & 10.4 & 9.7 \\
\textbf{\ours{}} & \textbf{26.9} & \textbf{24.8} \\
\bottomrule
\end{tabular}
\end{table}

\Cref{tab:coverage_profiles} reports the fraction of outlier channels covered by each calibration method at a fixed budget. \ours{} covers $26.9\%$ of all outlier channels compared to $16.8\%$ for random calibration, a $60\%$ relative improvement that widens to $69\%$ under the weighted metric ($24.8\%$ versus $14.7\%$). The rank ordering of methods on coverage matches their rank ordering on downstream accuracy in \Cref{tab:main_results_awq,tab:main_results_gptq}, and the magnitudes track as well. This is the specific empirical content of the paper's central structural claim: at fixed budget, methods that cover more weighted outlier mass quantize better.

The behavior of Max-ActVar is the puzzle posed by this table. Its criterion is plausible---samples with higher activation variance should include outlier channels---yet at $K = 128$ it covers only $10.4\%$ of outlier channels, less than random sampling. The explanation lies in redundancy: the average pairwise Jaccard similarity between coverage sets of selected samples is $J_{\text{Random}} = 0.34$, $J_{\text{Max-ActVar}} = 0.52$, and $J_{\ours{}} = 0.11$. Max-ActVar selections are individually informative but collectively overlapping, reactivating the same high-variance channels rather than spanning the outlier space, and this redundancy accounts for most of its empirical gap relative to \ours{}. A qualitative inspection of the samples \ours{} selects further shows that code, multilingual text, and mathematical notation activate outliers in different layer ranges; we defer this analysis to \Cref{app:qualitative}.

\section{Discussion}
\label{sec:discussion}

The algorithmic contribution of \ours{} is not sophistication but alignment with structure: once calibration selection is formalized as weighted coverage of outlier channels, greedy submodular maximization is the natural optimization tool, and its approximation guarantee follows from classical results. The substantive contribution lies upstream of the algorithm, in the formulation itself and in the surrogate consistency result that ties the specific weight choice to the per-channel clipping failure mode the method addresses. The practical cost of this approach is modest: a one-time offline profiling pass on the pool of roughly $15$ minutes on a single A100 for a 7B model, after which greedy selection is CPU-only and completes in under ten seconds for any calibration budget (\Cref{app:compute}); a detailed discussion of the formulation's limitations appears in \Cref{app:limitations}.

\section{Conclusion}
\label{sec:conclusion}

This paper formulated calibration data selection for post-training quantization as a weighted set cover problem over outlier channels. The associated greedy algorithm, \ours{}, enjoys a classical $(1-1/e)$ approximation guarantee on the induced coverage objective, and the specific weight choice is pinned by a surrogate consistency result tying missed weighted coverage to a stylized per-channel clipping model of PTQ error. Across three model families, two PTQ backends, and five downstream benchmarks, \ours{} improves over random, max-perplexity, max-activation-variance, and stratified baselines, with the largest gains at small budgets where random calibration is most likely to miss high-impact outlier channels. The analysis traces these gains to measured differences in outlier coverage and explains why the strongest heuristic baseline underperforms. The broader implication is that effective calibration is better understood as coverage of activation structure than as representativeness of data distributions.

\paragraph{Ethics statement.}
Better quantization improves accessibility and reduces deployment cost, but calibration choices can interact with subgroup performance in non-obvious ways. \ours{} increases the diversity of activation patterns in the calibration set, which mitigates but does not eliminate this risk; practitioners deploying quantized models should evaluate subgroup-specific behavior on their own deployment populations.

\bibliographystyle{plainnat}
\bibliography{references}

\begin{thebibliography}{32}
\providecommand{\natexlab}[1]{#1}
\providecommand{\url}[1]{\texttt{#1}}
\expandafter\ifx\csname urlstyle\endcsname\relax
  \providecommand{\doi}[1]{doi: #1}\else
  \providecommand{\doi}{doi: \begingroup \urlstyle{rm}\Url}\fi

\bibitem[Ashkboos et~al.(2024)Ashkboos, Mohtashami, Croci, Li, Cameron, Jaggi, Alistarh, Hoefler, and Hensman]{ashkboos2024quarot}
Saleh Ashkboos, Amirkeivan Mohtashami, Maximilian Croci, Bo~Li, Pashmina Cameron, Martin Jaggi, Dan Alistarh, Torsten Hoefler, and James Hensman.
\newblock {QuaRot}: Outlier-free 4-bit inference in rotated {LLMs}.
\newblock In \emph{NeurIPS}, 2024.

\bibitem[Bondarenko et~al.(2024)Bondarenko, Nagel, and Blankevoort]{bondarenko2024quantizable}
Yelysei Bondarenko, Markus Nagel, and Tijmen Blankevoort.
\newblock Quantizable transformers: Removing outliers by helping attention heads do nothing.
\newblock In \emph{ICLR}, 2024.

\bibitem[Chee et~al.(2023)Chee, Cai, Kuleshov, and De~Sa]{chee2023quip}
Jerry Chee, Yaohui Cai, Volodymyr Kuleshov, and Christopher De~Sa.
\newblock {QuIP}: 2-bit quantization of large language models with guarantees.
\newblock In \emph{NeurIPS}, 2023.

\bibitem[Chv{\'a}tal(1979)]{chvatal1979greedy}
Va{\v{s}}ek Chv{\'a}tal.
\newblock A greedy heuristic for the set-covering problem.
\newblock \emph{Mathematics of Operations Research}, 4\penalty0 (3):\penalty0 233--235, 1979.

\bibitem[Clark et~al.(2018)Clark, Cowhey, Etzioni, Khot, Sabharwal, Schoenick, and Tafjord]{clark2018arc}
Peter Clark, Isaac Cowhey, Oren Etzioni, Tushar Khot, Ashish Sabharwal, Carissa Schoenick, and Oyvind Tafjord.
\newblock Think you have solved question answering? {T}ry {ARC}.
\newblock \emph{arXiv:1803.05457}, 2018.

\bibitem[Dettmers et~al.(2022)Dettmers, Lewis, Belkada, and Zettlemoyer]{dettmers2022gpt3}
Tim Dettmers, Mike Lewis, Younes Belkada, and Luke Zettlemoyer.
\newblock {GPT3.int8()}: 8-bit matrix multiplication for transformers at scale.
\newblock In \emph{NeurIPS}, 2022.

\bibitem[Dettmers et~al.(2023)Dettmers, Svirschevski, Egiazarian, Kuznedelev, Frantar, Ashkboos, Borzunov, Hoefler, and Alistarh]{dettmers2023spqr}
Tim Dettmers, Ruslan Svirschevski, Vage Egiazarian, Denis Kuznedelev, Elias Frantar, Saleh Ashkboos, Alexander Borzunov, Torsten Hoefler, and Dan Alistarh.
\newblock {SpQR}: A sparse-quantized representation for near-lossless {LLM} weight compression.
\newblock \emph{arXiv:2306.03078}, 2023.

\bibitem[Frantar et~al.(2023)Frantar, Ashkboos, Hoefler, and Alistarh]{frantar2022gptq}
Elias Frantar, Saleh Ashkboos, Torsten Hoefler, and Dan Alistarh.
\newblock {GPTQ}: Accurate post-training quantization for generative pre-trained transformers.
\newblock In \emph{ICLR}, 2023.

\bibitem[He et~al.(2025)He, Xiao, and Han]{preservecalib2025}
Bowei He, Guangxuan Xiao, and Song Han.
\newblock Preserving {LLM} capabilities through calibration data curation: From analysis to optimization.
\newblock \emph{arXiv:2510.10618}, 2025.

\bibitem[Hendrycks et~al.(2021)Hendrycks, Burns, Basart, Zou, Mazeika, Song, and Steinhardt]{hendrycks2021mmlu}
Dan Hendrycks, Collin Burns, Steven Basart, Andy Zou, Mantas Mazeika, Dawn Song, and Jacob Steinhardt.
\newblock Measuring massive multitask language understanding.
\newblock In \emph{ICLR}, 2021.

\bibitem[Jiang et~al.(2023)Jiang, Sablayrolles, Mensch, Bamford, Chaplot, de~las Casas, Bressand, Lengyel, Lample, Saulnier, Lavaud, Lachaux, Stock, Le~Scao, Lavril, Wang, Lacroix, and El~Sayed]{jiang2023mistral}
Albert~Q. Jiang, Alexandre Sablayrolles, Arthur Mensch, Chris Bamford, Devendra~Singh Chaplot, Diego de~las Casas, Florian Bressand, Gianna Lengyel, Guillaume Lample, Lucile Saulnier, L{\'e}lio~Renard Lavaud, Marie-Anne Lachaux, Pierre Stock, Teven Le~Scao, Thibaut Lavril, Thomas Wang, Timoth{\'e}e Lacroix, and William El~Sayed.
\newblock Mistral 7b.
\newblock \emph{arXiv:2310.06825}, 2023.

\bibitem[Killamsetty et~al.(2021)Killamsetty, Sivasubramanian, Ramakrishnan, De, and Iyer]{killamsetty2021glister}
Krishnateja Killamsetty, Durga Sivasubramanian, Ganesh Ramakrishnan, Abir De, and Rishabh Iyer.
\newblock {GLISTER}: Generalization based data subset selection for efficient and robust learning.
\newblock In \emph{AAAI}, 2021.

\bibitem[Kim et~al.(2023)Kim, Hooper, Gholami, Dong, Li, Shen, Mahoney, and Keutzer]{kim2023squeezellm}
Sehoon Kim, Coleman Hooper, Amir Gholami, Zhen Dong, Xiuyu Li, Sheng Shen, Michael~W. Mahoney, and Kurt Keutzer.
\newblock {SqueezeLLM}: Dense-and-sparse quantization.
\newblock \emph{arXiv:2306.07629}, 2023.

\bibitem[Klemen and Robnik-{\v{S}}ikonja(2023)]{impactcalib2023}
Matej Klemen and Marko Robnik-{\v{S}}ikonja.
\newblock On the impact of calibration data in post-training quantization and pruning.
\newblock \emph{arXiv:2311.09755}, 2023.

\bibitem[Lin et~al.(2024)Lin, Tang, Tang, Yang, Dang, and Han]{lin2023awq}
Ji~Lin, Jiaming Tang, Haotian Tang, Shang Yang, Xingyu Dang, and Song Han.
\newblock {AWQ}: Activation-aware weight quantization for {LLM} compression and acceleration.
\newblock In \emph{MLSys}, 2024.

\bibitem[Liu et~al.(2024)Liu, Zhao, Iandola, Lai, Tian, Fedorov, Xiong, Chang, Shi, Krishnamoorthi, Lai, and Chandra]{liu2024spinquant}
Zechun Liu, Changsheng Zhao, Forrest Iandola, Chen Lai, Yuandong Tian, Igor Fedorov, Yunyang Xiong, Ernie Chang, Yangyang Shi, Raghuraman Krishnamoorthi, Liangzhen Lai, and Vikas Chandra.
\newblock {SpinQuant}: {LLM} quantization with learned rotations.
\newblock \emph{arXiv:2405.16406}, 2024.

\bibitem[{Meta AI}(2024)]{meta2024llama3}
{Meta AI}.
\newblock Llama 3 model card, 2024.

\bibitem[Mirzasoleiman et~al.(2020)Mirzasoleiman, Bilmes, and Leskovec]{mirzasoleiman2020coresets}
Baharan Mirzasoleiman, Jeff Bilmes, and Jure Leskovec.
\newblock Coresets for data-efficient training of machine learning models.
\newblock In \emph{ICML}, 2020.

\bibitem[Monaco et~al.(2026)Monaco, Vinzamuri, Choudhary, and Reddi]{calicuration2025}
Francesco~Pio Monaco, Bhanukiran Vinzamuri, Dhruv Choudhary, and Sashank~J. Reddi.
\newblock Frequency matters: Fast model-agnostic data curation for pruning and quantization.
\newblock \emph{arXiv:2603.16105}, 2026.

\bibitem[Nemhauser et~al.(1978)Nemhauser, Wolsey, and Fisher]{nemhauser1978analysis}
George~L. Nemhauser, Laurence~A. Wolsey, and Marshall~L. Fisher.
\newblock An analysis of approximations for maximizing submodular set functions---{I}.
\newblock \emph{Mathematical Programming}, 14\penalty0 (1):\penalty0 265--294, 1978.

\bibitem[Panferov et~al.(2024)Panferov, Borzunov, and Ushakov]{wiliamson2024diminishing}
Andrei Panferov, Alexander Borzunov, and Artem Ushakov.
\newblock Outliers and calibration sets have diminishing effect on quantization of modern {LLMs}.
\newblock \emph{arXiv:2405.20835}, 2024.

\bibitem[Sakaguchi et~al.(2020)Sakaguchi, Le~Bras, Bhagavatula, and Choi]{sakaguchi2020winogrande}
Keisuke Sakaguchi, Ronan Le~Bras, Chandra Bhagavatula, and Yejin Choi.
\newblock {WinoGrande}: An adversarial {W}inograd schema challenge at scale.
\newblock In \emph{AAAI}, 2020.

\bibitem[Sener and Savarese(2018)]{sener2018active}
Ozan Sener and Silvio Savarese.
\newblock Active learning for convolutional neural networks: A core-set approach.
\newblock In \emph{ICLR}, 2018.

\bibitem[Settles(2009)]{settles2009active}
Burr Settles.
\newblock Active learning literature survey.
\newblock Technical report, University of Wisconsin--Madison, 2009.

\bibitem[Shihab et~al.(2026)Shihab, Akter, and Sharma]{shihab2026fisher}
Ibne~Farabi Shihab, Sanjeda Akter, and Anuj Sharma.
\newblock Beyond variance: Knowledge-aware {LLM} compression via {F}isher-aligned subspace diagnostics.
\newblock \emph{arXiv:2601.07197}, 2026.

\bibitem[{Together AI}(2023)]{together2023redpajama}
{Together AI}.
\newblock Redpajama: An open dataset for training large language models, 2023.

\bibitem[Touvron et~al.(2023)Touvron, Martin, Stone, Albert, Almahairi, Babaei, Bashlykov, Batra, Bhargava, Bhosale, Bikel, Luber, Casas, Denoyer, Drozd, Elbaum, Esiobu, Ferrer, Goyal, and Hartshorn]{touvron2023llama2}
Hugo Touvron, Louis Martin, Kevin Stone, Peter Albert, Amjad Almahairi, Yasmine Babaei, Nikolay Bashlykov, Soumya Batra, Prajjwal Bhargava, Shruti Bhosale, Dan Bikel, Irene Luber, Juan Casas, Ludovic Denoyer, Oleksandr Drozd, Sergey Elbaum, David Esiobu, Cynthia~Breazeal Ferrer, Naman Goyal, and Graham Hartshorn.
\newblock Llama 2: Open foundation and fine-tuned chat models.
\newblock \emph{arXiv:2307.09288}, 2023.

\bibitem[Williams et~al.(2025)Williams, Chrysostomou, and Aletras]{selfcalib2024}
Miles Williams, George Chrysostomou, and Nikolaos Aletras.
\newblock Self-calibration for language model quantization and pruning.
\newblock In \emph{NAACL}, 2025.
\newblock arXiv:2410.17170.

\bibitem[Xiao et~al.(2023)Xiao, Lin, Seznec, Wu, Demouth, and Han]{xiao2023smoothquant}
Guangxuan Xiao, Ji~Lin, Mickael Seznec, Hao Wu, Julien Demouth, and Song Han.
\newblock {SmoothQuant}: Accurate and efficient post-training quantization for large language models.
\newblock In \emph{ICML}, 2023.

\bibitem[Xie et~al.(2023)Xie, Santurkar, Ma, and Liang]{xie2023data}
Sang~Michael Xie, Shibani Santurkar, Tengyu Ma, and Percy Liang.
\newblock Data selection for language models via importance resampling.
\newblock In \emph{NeurIPS}, 2023.

\bibitem[Zandieh et~al.(2026)Zandieh, Daliri, Hadian, and Mirrokni]{zandieh2025turboquant}
Amir Zandieh, Majid Daliri, Amin Hadian, and Vahab Mirrokni.
\newblock {TurboQuant}: Online vector quantization with near-optimal distortion rate.
\newblock In \emph{ICLR}, 2026.

\bibitem[Zellers et~al.(2019)Zellers, Holtzman, Bisk, Farhadi, and Choi]{zellers2019hellaswag}
Rowan Zellers, Ari Holtzman, Yonatan Bisk, Ali Farhadi, and Yejin Choi.
\newblock {HellaSwag}: Can a machine really finish your sentence?
\newblock In \emph{ACL}, 2019.

\end{thebibliography}

\appendix

\section{Extended Results: LLaMA-2-13B}
\label{app:13b}

\begin{table}[h]
\centering
\small
\caption{LLaMA-2-13B results under AWQ INT4 quantization with $K = 128$ calibration samples.}
\begin{tabular}{lccccc}
\toprule
\textbf{Calibration} & \textbf{MMLU} & \textbf{ARC-C} & \textbf{HellaSwag} & \textbf{WinoGr.} & \textbf{Wiki-PPL}$\downarrow$ \\
\midrule
FP16 & 55.0 & 59.4 & 82.1 & 76.0 & 4.88 \\
Random & 53.5 & 57.8 & 80.8 & 74.7 & 5.08 \\
Max-ActVar & 54.0 & 58.2 & 81.1 & 75.0 & 5.03 \\
\textbf{\ours{}} & \textbf{54.6} & \textbf{58.9} & \textbf{81.7} & \textbf{75.6} & \textbf{4.94} \\
\bottomrule
\end{tabular}
\end{table}

The pattern observed on the 7B and 8B models carries over to the 13B scale, with \ours{} recovering a comparable fraction of the FP16-to-INT4 accuracy gap across all evaluated benchmarks.

\section{GPTQ Backend Results}
\label{app:gptq}

\Cref{sec:main_results} reports AWQ results in full and summarizes the GPTQ results as exhibiting the same pattern. The full GPTQ numbers are given here.

\begin{table}[h]
\centering
\small
\caption{GPTQ INT4 results under the same setup as \Cref{tab:main_results_awq} ($K = 128$ calibration samples, three seeds). Best result per model in bold. For compactness we report the two strongest baselines (Random and Max-ActVar); Max-PPL and Stratified land between Random and Max-ActVar, mirroring the AWQ pattern.}
\label{tab:main_results_gptq}
\begin{tabular}{llccccc}
\toprule
\textbf{Model} & \textbf{Calibration} & \textbf{MMLU} & \textbf{ARC-C} & \textbf{HellaSwag} & \textbf{WinoGr.} & \textbf{Wiki-PPL}$\downarrow$ \\
\midrule
\multirow{3}{*}{LLaMA-2-7B}
& Random & 43.8 & 50.7 & 76.5 & 70.2 & 5.78 \\
& Max-ActVar & 44.3 & 51.1 & 76.9 & 70.6 & 5.72 \\
& \textbf{\ours{}} & \textbf{45.0} & \textbf{52.0} & \textbf{77.6} & \textbf{71.3} & \textbf{5.58} \\
\midrule
\multirow{3}{*}{LLaMA-3-8B}
& Random & 62.5 & 56.8 & 79.9 & 75.2 & 6.53 \\
& Max-ActVar & 63.2 & 57.1 & 80.3 & 75.7 & 6.46 \\
& \textbf{\ours{}} & \textbf{64.1} & \textbf{58.2} & \textbf{81.1} & \textbf{76.5} & \textbf{6.28} \\
\midrule
\multirow{3}{*}{Mistral-7B}
& Random & 60.2 & 59.0 & 81.3 & 74.9 & 5.57 \\
& Max-ActVar & 60.8 & 59.4 & 81.7 & 75.3 & 5.51 \\
& \textbf{\ours{}} & \textbf{61.6} & \textbf{60.3} & \textbf{82.4} & \textbf{76.0} & \textbf{5.36} \\
\bottomrule
\end{tabular}
\end{table}

The absolute accuracies under GPTQ are slightly lower than under AWQ for every method, reflecting GPTQ's greater sensitivity to activation statistics at INT4. The relative ordering of methods is identical across backends, and the magnitude of the \ours{}-versus-Random gap is comparable ($1.1$--$1.6$ MMLU points under GPTQ, versus $1.2$--$1.5$ under AWQ), which is the basis for the claim in \Cref{sec:main_results} that the benefit is backend-independent.

\section{Budget Efficiency: Full Curve}
\label{app:budget}

The summary in \Cref{sec:budget} reports the headline numeric comparison; here we include the full budget curve and corresponding table.

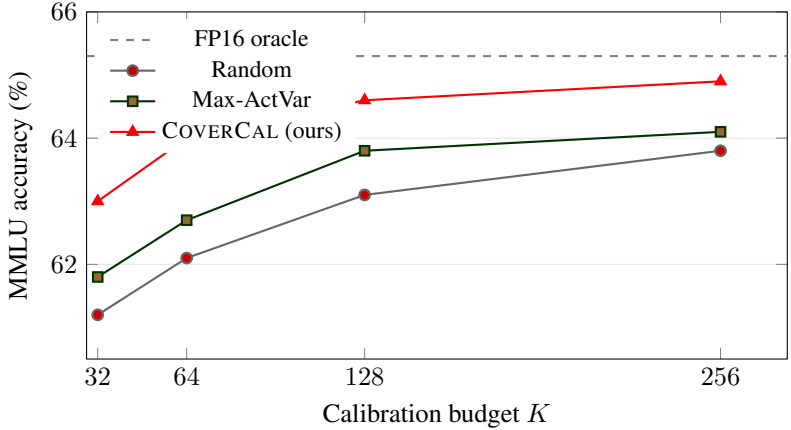
\begin{figure}[h]
\centering
\begin{tikzpicture}
\begin{axis}[
    width=0.78\textwidth, height=6.2cm,
    xlabel={Calibration budget $K$},
    ylabel={MMLU accuracy (\%)},
    xmin=28, xmax=280, ymin=60.5, ymax=66.0,
    xtick={32,64,128,256},
    xticklabels={$32$,$64$,$128$,$256$},
    ymajorgrids=true, grid style={gray!20},
    legend style={draw=none, at={(0.02,0.98)}, anchor=north west, fill=white, font=\footnotesize},
]
\addplot+[dashed, thick, gray, mark=none, domain=28:280] {65.3};
\addlegendentry{FP16 oracle}
\addplot+[mark=*, thick, mark size=2, color=black!60] coordinates {(32,61.2) (64,62.1) (128,63.1) (256,63.8)};
\addlegendentry{Random}
\addplot+[mark=square*, thick, mark size=1.8, color=black!80!green] coordinates {(32,61.8) (64,62.7) (128,63.8) (256,64.1)};
\addlegendentry{Max-ActVar}
\addplot+[mark=triangle*, thick, mark size=2.2, color=red] coordinates {(32,63.0) (64,64.0) (128,64.6) (256,64.9)};
\addlegendentry{\ours{} (ours)}
\end{axis}
\end{tikzpicture}
\caption{Budget efficiency on LLaMA-3-8B with AWQ INT4. \ours{} at $K = 64$ exceeds random calibration at $K = 256$, representing a fourfold reduction in the calibration budget required for equivalent MMLU.}
\label{fig:budget_curve}
\end{figure}

\begin{table}[h]
\centering
\caption{MMLU accuracy (\%) on LLaMA-3-8B with AWQ INT4, corresponding to \Cref{fig:budget_curve}.}
\label{tab:budget_curve}
\small
\begin{tabular}{lcccc}
\toprule
\textbf{Method} & $K=32$ & $K=64$ & $K=128$ & $K=256$ \\
\midrule
FP16 & \multicolumn{4}{c}{65.3} \\
Random & 61.2 & 62.1 & 63.1 & 63.8 \\
Max-ActVar & 61.8 & 62.7 & 63.8 & 64.1 \\
\textbf{\ours{}} & \textbf{63.0} & \textbf{64.0} & \textbf{64.6} & \textbf{64.9} \\
\bottomrule
\end{tabular}
\end{table}

The curve shows the expected coverage-hypothesis signature: the largest absolute gains appear at the smallest budget, and all three methods converge toward (but do not reach) the FP16 oracle as $K$ grows.

\section{Ablation Tables}
\label{app:ablations}

The summary in \Cref{sec:ablations} compresses two tables reported here in full.

\begin{table}[h]
\centering
\small
\caption{Weighting ablation on LLaMA-3-8B (AWQ INT4, $K = 128$). Removing either factor of \Cref{eq:channel_weight} leaves roughly half the gap from unweighted to fully-weighted coverage.}
\label{tab:weight_ablation}
\begin{tabular}{lcc}
\toprule
\textbf{Variant} & \textbf{MMLU} & \textbf{Wiki-PPL}$\downarrow$ \\
\midrule
Unweighted coverage & 64.0 & 6.31 \\
Magnitude-only weighting & 64.3 & 6.26 \\
Weight-sensitivity-only weighting & 64.1 & 6.29 \\
\textbf{Full weighting (\Cref{eq:channel_weight})} & \textbf{64.6} & \textbf{6.22} \\
\bottomrule
\end{tabular}
\end{table}

\begin{table}[h]
\centering
\small
\caption{Threshold sensitivity on LLaMA-3-8B (AWQ INT4, $K = 128$). The $|\cC|$ column reports the number of channels classified as outliers at each threshold. Performance is stable across the $5\sigma$--$7\sigma$ range.}
\label{tab:threshold_ablation}
\begin{tabular}{lccc}
\toprule
\textbf{Threshold} & $|\cC|$ & \textbf{MMLU} & \textbf{Wiki-PPL}$\downarrow$ \\
\midrule
$4\sigma$ & 12{,}340 & 64.3 & 6.27 \\
$5\sigma$ & 7{,}820 & 64.4 & 6.24 \\
$\mathbf{6\sigma}$ & 4{,}950 & \textbf{64.6} & \textbf{6.22} \\
$7\sigma$ & 2{,}100 & 64.2 & 6.28 \\
$8\sigma$ & 890 & 63.7 & 6.35 \\
\bottomrule
\end{tabular}
\end{table}

\section{Qualitative Structure of Selected Samples}
\label{app:qualitative}

Inspection of the calibration sets produced by \ours{} reveals interpretable structure in the selected samples. Code-heavy text activates outlier channels concentrated in early attention layers, where token-level syntactic distinctions drive activation patterns that differ sharply from natural-language input. Multilingual text activates outliers concentrated in middle layers, consistent with the position at which cross-lingual representations are known to diverge from the monolingual manifold in transformer models. Mathematical notation activates outliers in later layers associated with specialized reasoning subcircuits. What is notable about this distribution is that it emerges purely from optimizing the coverage objective, without any explicit domain-balancing heuristic, which suggests that outlier coverage is functioning as a principled proxy for calibration diversity rather than substituting for it. The selection mechanism is not ``pick one code sample, one math sample, one multilingual sample'' but rather ``pick samples whose collective activation pattern spans the outlier space''; that the selections then correspond to recognizable domain categories is an emergent property of the outlier structure itself.

\section{Computational Cost}
\label{app:compute}

The computational overhead of \ours{} is small enough to be practically negligible relative to PTQ itself.

\begin{table}[h]
\centering
\small
\caption{Wall-clock time for calibration selection on LLaMA-3-8B with pool size $N = 10{,}000$ and budget $K = 128$.}
\label{tab:compute_cost}
\begin{tabular}{lcc}
\toprule
\textbf{Phase} & \textbf{Hardware} & \textbf{Wall-clock time} \\
\midrule
Offline profiling (one-time) & $1\times$ A100 & $\sim 15$ min \\
Greedy selection & CPU only & $< 10$ s \\
\bottomrule
\end{tabular}
\end{table}

The one-time offline profiling pass takes approximately $15$ minutes on a single A100, and greedy selection itself is CPU-only and completes in under ten seconds; its cost is dominated by cache-friendly sparse-matrix aggregation over the pre-computed coverage matrix. The total overhead relative to random calibration is therefore approximately $15$ minutes, amortized across all subsequent calibration runs for the same model, since the profiling output depends on the model architecture but not on the downstream task or the calibration budget.

\section{Outlier Channel Statistics}
\label{app:outlier_stats}

\begin{table}[h]
\centering
\small
\caption{Outlier channel statistics at the $6\sigma$ threshold. The total count is the number of (layer, channel) pairs in the model; $|\cC|$ is the number identified as outlier channels under the pool.}
\begin{tabular}{lccc}
\toprule
\textbf{Model} & \textbf{(Layer, channel) pairs} & \textbf{$|\cC|$} & \textbf{Fraction outlier} \\
\midrule
LLaMA-2-7B & 131{,}072 & 4{,}230 & $3.2\%$ \\
LLaMA-3-8B & 131{,}072 & 4{,}950 & $3.8\%$ \\
Mistral-7B & 131{,}072 & 3{,}680 & $2.8\%$ \\
LLaMA-2-13B & 204{,}800 & 6{,}120 & $3.0\%$ \\
\bottomrule
\end{tabular}
\end{table}

Between $2.8\%$ and $3.8\%$ of (layer, channel) pairs are classified as outlier channels at the $6\sigma$ threshold across the four models. The narrow range is consistent with the interpretation that outlier structure is a stable architectural property of modern transformers rather than an artifact of any particular training run.

\section{Adaptive Threshold Refinement}
\label{app:adaptive}

\begin{algorithm}[h]
\caption{Adaptive threshold refinement.}
\begin{algorithmic}[1]
\REQUIRE Initial threshold $\tau_0 = 6\sigma$, model $f_\theta$, pool $\cP$, budget $K$.
\STATE $\cS_0 \leftarrow \ours{}(\cP, \tau_0, K)$
\STATE $\hat{f} \leftarrow \mathrm{PTQ}(f_\theta, \cS_0)$
\FOR{each layer $l$}
    \STATE $\epsilon^{(l)} \leftarrow \bigl\|\mathbf{W}^{(l)}\mathbf{X}^{(l)} - \hat{\mathbf{W}}^{(l)}\mathbf{X}^{(l)}\bigr\|_F$
    \IF{$\epsilon^{(l)} > \mathrm{median}\bigl(\{\epsilon^{(l')}\}_{l'}\bigr) + 2\sigma_\epsilon$}
        \STATE $\tau^{(l)} \leftarrow 4\sigma$
    \ENDIF
\ENDFOR
\STATE $\cS_1 \leftarrow \ours{}(\cP, \{\tau^{(l)}\}_l, K)$
\RETURN $\cS_1$
\end{algorithmic}
\end{algorithm}

\section{Discussion of the Surrogate Consistency Result}
\label{app:surrogate}

\Cref{prop:surrogate} rests on two modeling assumptions, (M1) and (M2), each of which merits separate discussion. Under (M1), a covered channel incurs zero surrogate deficit. Covering a channel in the sense of \Cref{eq:outlier_score} means that at least one calibration sample activates the channel above the outlier threshold, so the quantization scale derived from the calibration set captures its dynamic range up to the observed maximum. Setting $\delta_{l,c}(\cS) = 0$ in the stylized model is an optimistic idealization: in practice the deficit is small rather than exactly zero, for two reasons. First, inference may see activations larger than the calibration maximum, in which case clipping can still occur. Second, even when the observed range is sufficient, quantization introduces rounding error that is bounded but nonzero. The surrogate absorbs both residual error sources into the zero term, which means the result models the covered case optimistically.

Under (M2), an uncovered channel has deficit bounded by its normalized reference magnitude. When an outlier channel is not activated during calibration, the quantization scale is set from the non-outlier range and the channel is clipped at inference. The bound $\delta_{l,c}(\cS) \leq o_c^{(l,\mathrm{ref})}/\tau^{(l)}$ encodes the intuition that the clipping-induced error is controlled by the channel's normalized reference magnitude, which is precisely what the weight $w_{l,c}$ charges for in \Cref{eq:channel_weight}. The bound is tight in the worst case but may overstate the deficit when the inference-time activation is smaller than the reference-pool maximum.

The surrogate loss itself is a per-channel decomposition in which each channel contributes independently, weighted by its column norm. The Frobenius loss of \Cref{eq:ptq_objective}, in contrast, is the norm of a full matrix-matrix product and therefore involves cross-channel interactions through the weight matrix. We do not claim that the surrogate upper-bounds the Frobenius loss; doing so would require additional assumptions about the geometry of $\mathbf{W}^{(l)}$ and the conditioning of the quantizer. The surrogate is best understood as isolating the single failure mode \ours{} is designed to target, namely per-channel clipping due to missed outlier coverage, and showing that the chosen weighted objective is the right one for that failure mode.

The appropriate reading of \Cref{prop:surrogate} is therefore the following. The result formalizes why \Cref{eq:channel_weight} is the weight to use in \Cref{eq:set_cover_max}: it is exactly the per-channel quantity that bounds the surrogate loss under (M1)--(M2), and both factors of the weight are pinned by this relation. It is not a theorem about downstream accuracy, and the empirical relationship between $F$ and downstream metrics remains the content of \Cref{sec:experiments,sec:analysis}.

\section{Limitations}
\label{app:limitations}

The discussion in \Cref{sec:discussion} summarizes the paper's scope; we record here the specific caveats this scope implies.

The outlier threshold $\tau^{(l)} = \mu + 6\sigma$ is heuristic, and while the ablation in \Cref{app:ablations} shows stability across $5\sigma$--$7\sigma$, architectures with qualitatively different outlier distributions may require different thresholds; a principled per-model calibration of $\tau^{(l)}$ would be a natural extension. \ours{} further assumes that outlier channels are a major driver of PTQ error, an assumption supported here and by prior work on outlier features but one that may weaken for models trained with quantization-aware pre-training or other outlier-suppressing regimes. The quality of coverage also depends on the candidate pool: a pool drawn from a narrow domain may simply not contain samples that activate particular outlier channels, and no selection method can overcome this limitation from within the pool.

The evaluation itself is bounded in two ways. We report INT4 quantization only; at more aggressive bit-widths (INT3 or below) the margin for calibration error shrinks, and we expect the adaptive-threshold variant of \Cref{app:adaptive} to become more valuable there. And the theoretical guarantees apply to the weighted coverage objective $F$ (via \Cref{thm:guarantee}) and to the clipping surrogate $\mathcal{L}_{\mathrm{sur}}$ (via \Cref{prop:surrogate}); neither extends directly to the Frobenius reconstruction loss of \Cref{eq:ptq_objective} or to downstream accuracy. The empirical link between $F$ and downstream metrics is established in \Cref{sec:experiments,sec:analysis}, and a tighter formal connection would require additional assumptions on the quantizer and the model that are not standard in the current PTQ literature.

\section{Per-Layer Reconstruction Error Breakdown}
\label{app:perlayer}

\begin{table}[h]
\centering
\small
\caption{Per-layer Frobenius reconstruction error on LLaMA-3-8B (AWQ INT4, random calibration at $K = 128$). The uncovered column reports the reconstruction error when the calibration set fails to activate the three most extreme outlier channels in layer 14 (channels 1847, 2903, 4011, with activation magnitudes above $120\sigma$). The covered column reports the error under an otherwise identical calibration set augmented with samples that activate those channels. The construction isolates the contribution of outlier coverage at a single layer.}
\label{tab:perlayer_error}
\begin{tabular}{lcccc}
\toprule
\textbf{Layer group} & \textbf{Uncovered $\epsilon$} & \textbf{Covered $\epsilon$} & \textbf{\% of total (uncov.)} & \textbf{Reduction} \\
\midrule
Layers 0--7 & 0.42 & 0.41 & $5.8\%$ & $2\%$ \\
Layers 8--13 & 1.18 & 1.15 & $16.3\%$ & $3\%$ \\
\textbf{Layer 14} & \textbf{4.44} & \textbf{0.22} & \textbf{$61.3\%$} & \textbf{$95\%$} \\
Layers 15--23 & 0.87 & 0.85 & $12.0\%$ & $2\%$ \\
Layers 24--31 & 0.33 & 0.32 & $4.6\%$ & $3\%$ \\
\midrule
\textbf{Total} & \textbf{7.24} & \textbf{2.95} & $100\%$ & $59\%$ \\
\bottomrule
\end{tabular}
\end{table}

Under random calibration, layer $14$ alone contributes $61.3\%$ of the aggregate Frobenius reconstruction error, and three channels within it (channels 1847, 2903, 4011, all with activation magnitudes above $120\sigma$) account for that concentration. When calibration is augmented with samples that activate these specific channels, the layer-14 reconstruction error falls by $95\%$ and the aggregate error by $59\%$. We note that Frobenius reconstruction error is a proxy for downstream accuracy loss, and the relationship between the two is monotone but not linear, so the $61.3\%$ error share should not be read as a $61.3\%$ share of downstream accuracy degradation. The table nonetheless establishes the mechanism underlying the paper's central claim: a small number of outlier channels can dominate reconstruction error when uncovered, and coverage of those specific channels recovers most of the lost quality.

\end{document}